\documentclass[10pt,twocolumn,letterpaper]{article}

\usepackage{cvpr}
\usepackage[utf8]{inputenc}
\usepackage{graphicx}
\usepackage{color}
\usepackage{times}
\usepackage{amsmath,mathtools}
\usepackage{amsthm}	
\usepackage{amssymb}
\usepackage{enumerate}
\usepackage{subfigure}
\usepackage{algorithm} 
\usepackage{algpseudocode}
\usepackage{booktabs}
\usepackage{tabularx}
\usepackage{array}
\usepackage{bm}
\usepackage{url}
\usepackage{bbm}
\usepackage{enumitem}
\usepackage{stackengine}
\stackMath

\usepackage[pagebackref=true,breaklinks=true,letterpaper=true,colorlinks,bookmarks=false]{hyperref}

\cvprfinalcopy 

\usepackage{tabularx}
\newcommand{\PBS}[1]{\let\temp=\\#1\let\\=\temp}
\newcommand{\RBS}{\let\\=\tabularnewline}

\DeclareMathOperator{\atantwo}{atan2}
\renewcommand{\vec}[1]{\mathbf{#1}}

\newtheoremstyle{break}
  {\topsep}{\topsep}%
  {\itshape}{}%
  {\bfseries}{}%
  {\newline}{}%

\theoremstyle{break}
\newtheorem{definition}{Definition}

\newtheorem{theorem}{Theorem}

\newcommand{\insertimageC}[5]{ 
\begin{figure}[#5]
\centering
\includegraphics[width=#1\linewidth, clip=true]{figures/#2}
\caption{#3}
\label{#4}
\end{figure}
}

\newcommand{\insertimageStar}[5]{ 
\begin{figure*}[#5]
\centering
\includegraphics[width=#1\linewidth, clip=true]{figures/#2}
\caption{#3}
\label{#4}
\end{figure*}
}

\algnewcommand\algorithmicinput{\textbf{Input:}}
\algnewcommand\INPUT{\item[\algorithmicinput]}
\algnewcommand\algorithmicoutput{\textbf{Output:}}
\algnewcommand\OUTPUT{\item[\algorithmicinput]}

\newcommand{\comment}[1]{}

\newcommand{\suchthat}{\;\ifnum\currentgrouptype=16 \middle\fi|\;}

\newcommand\SmallMatrix[1]{{%
  \normalsize\arraycolsep=0.44\arraycolsep\ensuremath{\begin{bmatrix}#1\end{bmatrix}}}}

\DeclareMathSymbol{@}{\mathord}{letters}{"3B}

\def\cvprPaperID{1020} 

\ifcvprfinal\fi
\begin{document}

\title{A Minimalist Approach to Type-Agnostic Detection of Quadrics in Point Clouds}

\author{Tolga Birdal$^{1,2}$ \quad Benjamin Busam$^{1,3}$ \quad Nassir Navab$^{1,4}$ \quad Slobodan Ilic$^{1,2}$ \quad Peter Sturm$^{5}$ \\
$^1$ Technische Universit{\"a}t M{\"u}nchen, Germany\qquad
$^2$ Siemens AG, Munich, Germany\\
\hspace{-3pt}$^3$ Framos AG, Munich, Germany
\qquad  %
$^4$ Johns Hopkins University, US\qquad
$^5$ INRIA, France\\%
}

\maketitle

\begin{abstract}
This paper proposes a segmentation-free, automatic and efficient procedure to detect general geometric quadric forms in point clouds, where clutter and occlusions are inevitable.
Our everyday world is dominated by man-made objects which are designed using 3D primitives (such as planes, cones, spheres, cylinders, etc.).
These objects are also omnipresent in industrial environments. This gives rise to the possibility of abstracting 3D scenes through primitives, thereby positions these geometric forms as an integral part of perception and high level 3D scene understanding.

As opposed to state-of-the-art, where a tailored algorithm treats each primitive type separately, we propose to encapsulate all types in a single robust detection procedure. At the center of our approach lies a closed form 3D quadric fit, operating in both primal \& dual spaces and requiring as low as 4 oriented-points. Around this fit, we design a novel, local null-space voting strategy to reduce the 4-point case to 3. Voting is coupled with the famous RANSAC and makes our algorithm orders of magnitude faster than its conventional counterparts. This is the first method capable of performing a generic cross-type multi-object primitive detection in difficult scenes. Results on synthetic and real datasets support the validity of our method.
\end{abstract}

\section{Introduction}
Quadrics, or quadratic surfaces, are second order implicit representations, including geometric primitives such as planes, spheres, cylinders, ellipsoids, cones and more.
Due to their ability to cover a wide variety of shapes, many industrial parts and man-made objects are manufactured using quadric solids after designed with specific CAD software. As a result, our environment consists of these mathematical constructs.
This prominent exposure made understanding and reverse engineering of quadrics, where a scene is parsed into primitives, a primary concern in computer vision already in the 80s~\cite{miller1988,birdal2017iccv}. While a large part of these works address the problem of recovering quadrics from 2D images \cite{Cross1998}, in the 3D domain, the research streams still seem to be divided into two branches: \textit{Quadric Fitting} and \textit{Primitive Detection}.
The former, \textit{Quadric Fitting}, concerns the fitting of a general quadric to a clutter-free, potentially noisy scene \cite{Taubin1991,Tasdizen2001,Blane2000}. The latter, \textit{Primitive Detection}, tackles a different problem: the reconstruction of type-specific primitives from cluttered scenes \cite{Allaire2007,Andrews2013,georgiev2016real}. Yet, to this date, the task of automatic and generic quadric detection within clutter or occlusions without necessitating auxiliary preprocessing steps such as segmentation or manual intervention \cite{morwald2013geometric}, remain to be unexplored and unsolved. This is partially because unlike rigid 6-DoF free-form detection $\&$ pose estimation, quadrics have 9 degrees of freedom, and can deform into various forms. Naturally, this problem is a harder one. Although deep neural networks can learn to segment scenes \cite{minto2016scene,garcia2017review}, or show great potential in learning the fitting function and feature extraction~\cite{ppfnet}, jointly solving the detection and surface fitting, to the best of our knowledge, is an open challenge. Moreover, for the problem at hand, data labeling is notoriously exhaustive.

\insertimageC{1}{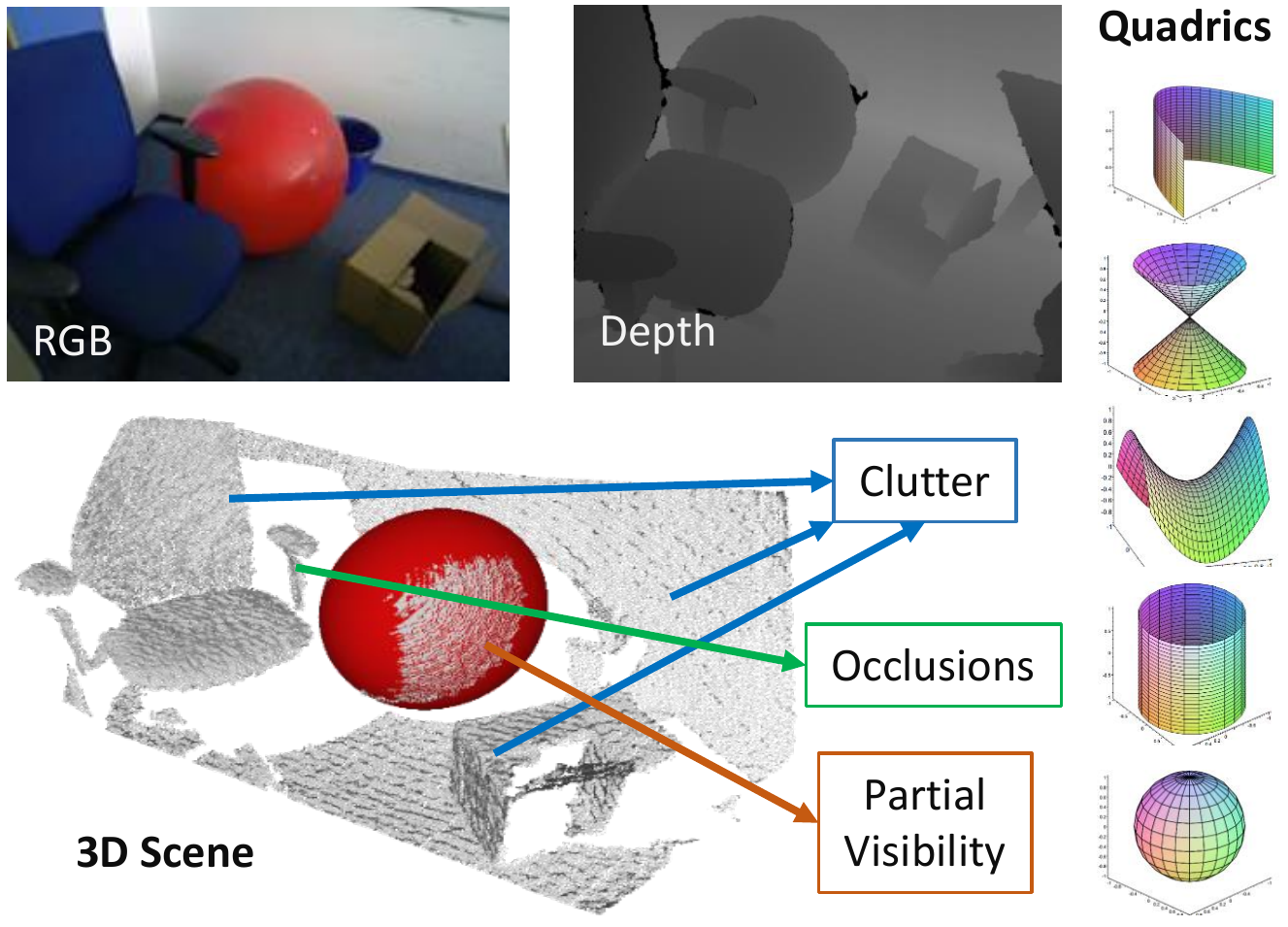}{Our algorithm quickly detects quadric shapes in point clouds under occlusions, partial visibility and clutter. Note that, the algorithm has no clue whether a sphere exists or not and is able to detect the plausible primitive using only 3D geometry.}{fig:teaser}{t!}

One trivial way to attack our problem would be to use an off-the-shelf quadric fitting method in a RANSAC~\cite{Fischler1981} pipeline. Yet, there is a caveat in doing this: RANSAC requires repeated trials of a minimal-case model fitting. Even though degenerate or special quadrics such as cylinders, planes or spheres can be described by three or two or single oriented 3D point(s)~\cite{drost2015local}, it is not the case for a more complex quadric such as an ellipsoid, paraboloid or hyperboloid and the state of the art requires 9 ~\cite{Taubin1991,Tasdizen2001,Blane2000} to 12 \cite{beale2016fitting} points to perform a generic quadric fit. Such large DoF necessity increases the complexity of RANSAC-like approaches significantly. Hence, this seemingly old problem - discovering equations of quadrics in arbitrary and general 3D scenes - awaits an efficient and robust clutter treatment \cite{Petitjean2002,Frahm2006}.

In this paper, we fill these gaps by showing first a mixed primal/dual-space quadric fit, requiring as low as 4 oriented-points (4 points and associated normals). We analyze this fit and around it, devise a local Hough transform method to reduce the 4-point case to 3-points, establishing the minimalist fit known so far. These developments pave the way to an efficient RANSAC loop. We slightly modify the RANSAC pipeline using a pre-clustering of feasible hypotheses. We  provide distance and score metrics on quadric spaces to perform this robust fit.
In a nutshell, our contributions are:
\begin{enumerate}[noitemsep]
	\item A novel closed form quadric fitting capable of using as low as 4 oriented points as supported theoretically,
	\item An effective local voting strategy, reducing the 4-point requirement to 3-points,
	\item A fast \& robust scheme to compute voting parameters,
	\item A 3-point RANSAC algorithm with an efficient score function to verify the hypotheses,
	\item To the best of our knowledge, the first robust pipeline to perform generic type-agnostic multi-object quadric detection in the presence of clutter and occlusion.
\end{enumerate}

Quantitative and qualitative experiments on synthetic and real datasets demonstrate the robustness of our fitting as well as the computational gains. We will make our implementation openly available.
\section{Prior Art}
\label{sec:relatedWork}
\paragraph{Quadrics} Quadrics appear in various domains of vision, graphics and robotics. They are found to be one of the best local surface approximators in estimating differential properties \cite{Petitjean2002}. Thus, point cloud normals and curvatures are oftentimes estimated via local quadrics~\cite{Zhao20163,Digne2016sparse,birdal2015point,di2015geometric}. Other use cases include mesh representation~\cite{guennebaud2007algebraic}, mesh and point-cloud segmentation~\cite{yan2006quadric,liu2016robust}, projective geometry $\&$ reconstruction ~\cite{Gay2017ICCV,Cross1998}, minimal problem solving~\cite{kukelova2016efficient}, grasp motion planning~\cite{Pas2013,Uto2013} and feature extraction~\cite{You2017}.
\vspace{-3mm}
\paragraph{Primitive Detection} Discovering primitives in point clouds has kept the vision researchers busy for a long period of time. This category treats the primitive shapes independently~\cite{georgiev2016real}, giving rise to plane, sphere, cone, cylinder and etc. specific fitting algorithms. Planes, as the simplest forms, are the primary targets of Hough-family~\cite{borrmann20113d}. Yet, to find a larger variety of primitives, RANSAC comes in handy as shown in the prosperous Globfit~\cite{li2011globfit}: a relational local to global RANSAC algorithm. Schnabel et. al.~\cite{Schnabel2007} and Tran et. al.~\cite{Tran2015} also focus on reliable estimation via RANSAC. Qui et.al. extract pipe runs using cylinder fitting \cite{Qiu2014}, and the local Hough transform of Drost and Ilic \cite{drost2015local} efficiently detects spheres, cylinders and planes from point clouds. Lopez et. al.~\cite{lopez2016robust} devise a robust ellipsoid fitting based on iterative non-linear optimization while Sveier et. al.~\cite{sveier2017object} suggest a conformal geometric algebra to spot planes, cylinders and spheres. Andrews~\cite{Andrews2014} deals with paraboloids and hyperboloids in CAD models.\\
Methods and algorithms in this category are quite successful in shape detection, yet they handle the primitives separately. This prevents automatic type detection, or generalized modeling of surfaces such as paraboloids or hyperboloids along with cylinders and spheres.
\vspace{-3mm}
\paragraph{Quadric Fitting} 
Since the 90s, following the pioneering work of Gabriel Taubin \cite{Taubin1991}, quadric fitting is cast as a constrained optimization problem, where the solution is obtained from a Generalized Eigenvalue decomposition of a scatter matrix. 
This work has then been enhanced by 3L fitting~\cite{Blane2000}, proposing a local, explicit ribbon surface composed of three-level-sets of constant Euclidean distance to the object boundary. This fit implicitly uses the local surface information. Later on Tasdizen \cite{Tasdizen2001} improved the local surface properties by incorporating the surface normals to regularize the scatter matrix, which allows for a good and stable fit. Recently, Beale et. al.~\cite{beale2016fitting} introduced the use of a Bayesian prior for the fitting pipeline. All of these methods use at least 9 points. Moreover, they only use surface normals as regularizers and not as additional constraints or are unable to deal with outliers in data. There are a few other studies~\cite{kanatani2005further,Allaire2007}, which improve these standard methods, but many of them either involve non-linear optimization \cite{yan2012variational} or they share the common drawback of requiring 9 independent constraints and no outlier treatment.
\vspace{-3mm}
\paragraph{Quadric Detection} Recovering general quadratic forms from cluttered and occluded scenes is a rather unexplored area. A promising direction was to represent quadrics with spline surfaces ~\cite{morwald2013geometric}, but such approaches have to tackle the increased number of control points, i.e. 8 for spheres, 12 for general quadrics \cite{qin1997representing,qin1998representing}. Segmentation is one way to overcome such difficulties \cite{makhal2017grasping,morwald2013geometric}, but introduces undesired errors especially under occlusions. Other works exploit genetic algorithms \cite{gotardo2004robust} but have the obvious drawback of inefficiency. QDEGSAC~\cite{frahm2006ransac} proposed a 6-point hierarchical RANSAC, but the paper misses out an evaluation or method description for a quadric fit. Petitjean~\cite{Petitjean2002} stressed out the necessity of outlier aware quadric fitting however only ends up suggesting M-estimators for future research. 
\section{Preliminaries}
\label{sec:description}
\vspace{-1mm}
\begin{definition}
A \textbf{quadric} in 3D Euclidean space is a hypersurface defined by the zero set of a polynomial of degree two:
\begin{align}
\label{eq:quadric}
f (x,y,z)
&= Ax^2+ By^2 + Cz^2 + 2Dxy + 2Exz \\
& + 2Fyz + 2Gx + 2Hy + 2Iz + J = 0. \nonumber
\end{align}
\begin{align}
&\vec{q} = \begin{bmatrix}
           A & B & C & D & E & F & G & H & I & J
         \end{bmatrix}^T \\
&\vec{v} = \begin{bmatrix}
			x^2 & y^2 & z^2 & 2xy & 2xz & 2yz & 2x & 2y & 2z & 1
         \end{bmatrix}^T \nonumber
\end{align}
\end{definition}
Using homogeneous coordinates, quadrics can be analyzed uniformly.
The point $\vec{x} = (x,y,z) \in \mathbb{R}^3$ lies on the quadric, if the projective algebraic equation over $\mathbb{RP}^3$ with $d_q(\vec{x}) := [\vec{x}^T  1]\vec{Q} [\vec{x}^T 1]^T$ holds true, where the matrix $\vec{Q} \in \mathbb{R}^{4 \times 4}$ is defined by re-arranging the coefficients:
\begin{equation}
\vec{Q} =
\SmallMatrix{
A & D & E & G\\
D & B & F & H\\
E & F & C & I\\
G & H & I & J
}
,\quad
\nabla \vec{Q} =  2
\SmallMatrix{
A & D & E & G \\
D & B & F & H \\
E & F & C & I
}.
\end{equation}
$d_q(\vec{x})$ can be viewed as an algebraic distance function. Similar to the quadric equation, the gradient at a given point can be written as $\nabla\vec{Q}(\vec{x}) := \nabla\vec{Q} [\vec{x}^T 1]^T$.
Quadrics are general implicit surfaces capable of representing cylinders, ellipsoids, cones, planes, hyperboloids, paraboloids and potentially the shapes interpolating any two of those.
All together there are 17 sub-types~\cite{Weisstein2017}.
Once $\vec{Q}$ is given, this type can be determined from an eigenvalue analysis of $\vec{Q}$ and its subspaces.
Note that quadrics have constant second order derivatives and are practically smooth.
\vspace{-1mm}
\begin{definition}{A \textbf{dual quadric} $\mathbf{Q}^*\sim \mathbf{Q}^{-1}$ is the locus of all planes $\{\mathbf{\Pi}_i\}$ satisfying $\mathbf{\Pi}^T_i\mathbf{Q}^{-1}\mathbf{\Pi}_i=0$} and is formed by the Legendre transform, sending points to tangent planes as covectors.
\end{definition}
\vspace{-1mm}
\noindent Knowing a point lies on the surface gives 1 constraint, and if, in addition, one knows the tangent plane at that point, then one gets 2 more constraints. Such view will lead to a mixed-space fitting, reducing the minimal point necessity.\vspace{-1mm}
\begin{definition}{A \textbf{basis} $\vec{b}$ is a subset composed of a fixed number ($b$) of scene points and hypothesized to lie on the sought surface.}
\end{definition}\vspace{-1mm}
\section{Quadric Detection in Point Clouds}
\label{sec:detection}
\vspace{-1mm}
Our algorithm operates by repeatedly selecting bases. Once a basis $\vec{b}$ is fixed, an under-determined quadric fit parameterizes the solution and a local accumulator space, tied to $\vec{b}$, is formed. All other points in the scene are then combined with this basis to vote for the potential generic primitive. To discover the optimal basis, we perform RANSAC, iteratively hypothesizing solution candidates. Subsequent to such joint RANSAC and voting, a hypotheses custering and ranking takes place, allowing for multiple quadric detection without re-executions of the whole algorithm. 
\vspace{-3mm}
\paragraph{A new perspective to quadric fitting}
Direct solutions for quadric fitting rely either solely on point sets~\cite{Taubin1991}, or use normals as regularizers~\cite{Tasdizen2001}. This is due to vector-vector alignment being a non-linear constraint caused by the normalization $\nabla \vec{Q}(\vec{x}) / \|\nabla \vec{Q}(\vec{x})\|$, i.e. it is hard to know the norm in advance. Similar to $\nabla1$-fitting \cite{Tasdizen1999,Tasdizen2001}, our idea is to align the gradient vector of the quadric $\nabla \vec{Q}(\vec{x}_i)$ with surface normal $\vec{n}_i \in \mathbb{R}^3$, this time using a linear gradient constraint to increase the DoFs rather than regularizing the solution. 
To do so, we introduce per normal homogeneous scale $\alpha_i$ among the unknowns and write: $\nabla \vec{v}^T_i \vec{q} = \alpha_i \vec{n}_i$. Stacking it up for all $N$ points $\vec{x}_i$ and normals $\vec{n}_i$ leads to:
\vspace{-1mm}
\begin{equation}
\vec{A}^\prime
\begin{bmatrix}
\vec{q}\\
\bm{\alpha}
\end{bmatrix}=
\left[\vcenter{\hbox{\addstackgap[1pt]{%
\stackanchor{
\Bigl[\makebox[13ex]{$\cdots\,\vec{v}_i\,\cdots$}\Bigr]^T
\Bigl[\makebox[10ex]{$\vec{0}$}\Bigr]}
{
\Bigl[\makebox[13ex]{$\cdots\,\nabla\vec{v}_i\,\cdots$}\Bigr]^T~
\Bigl[\makebox[10ex]{$\text{\textit{diag}}(\vec{-n}^\prime)$}\Bigr]}
}
}}\right]
\renewcommand{\arraystretch}{1.75}
\begin{bmatrix}
\vec{q}\\
\bm{\alpha}\\
\end{bmatrix}=\vec{0}
\label{eq:systemfull}
\end{equation}
where $\vec{A}^\prime$ is $4N\times(N+10)$, $\bm{\alpha}$ are the unknown homogeneous scales and $\vec{n}^\prime=[\cdots\vec{n}_i^T\cdots]^T$. The solution defines a unique fit in the null space of $\vec{A}^\prime$, which gets to full rank for $\geq 4$ oriented points, the minimalist case known up to now.
\vspace{-1mm}
\begin{theorem}
Although 3 oriented points provide 9 constraints, quadric fitting, as formulated, possesses a trivial solution in addition to the true one.\vspace{-2mm}
\end{theorem}
\begin{proof}
Let us call \textbf{data-plane}, the plane spanned by the coordinates of 3 data points. Any rank-1 quadric consists of a single plane $\mathbf{\Pi}$ and can be written as $\vec{Q} = \mathbf{\Pi} \; \mathbf{\Pi}^T$. Hence, for any point $\vec{u}$ on the data plane and thus on the quadric, we have $\vec{Q} \vec{u} = \vec{0}$ and $\nabla\vec{v}_i^T\vec{q}=\vec{0}$. This amounts to the choice $\bm{\alpha}=\vec{0}$. Thus the trivial solution is identified as the rank-1 quadric consisting of the data-plane and zero scale factors. Hence, the estimation problem admits at least a one-dimensional linear family of solutions, spanned by the true quadric and the rank-1 quadric of the data-plane. 
\vspace{-1mm} 
\end{proof}
Indeed, for a non-degenerate surface, following relations hold: $N = 1 \Rightarrow \text{rk}(\vec{A}^\prime) = 4$, $N = 2 \Rightarrow \text{rk}(\vec{A}^\prime) = 7$, $N = 3 \Rightarrow \text{rk}(\vec{A}^\prime) = 9$ and $N > 3 \Rightarrow \text{rk}(\vec{A}^\prime) = 10$. Showing this and that adding further point constraints give diminishing returns is also relatively easy to see by Gaussian Elimination on the matrix, as equations are inter-dependent.

Despite sparsity, the unknowns in this system scale linearly with $N$, leading to a large system to solve. In practice, we approximate them with a single common scale $\bar{\alpha}$, similar to adding a \textbf{soft-regularizer} that tries to force $\alpha_i \gets \bar{\alpha}$~\cite{Tasdizen2001}.
Doing so rescues us from solving the sensitive homogeneous system~\cite{You2017}, and lets us re-write the system in a compact form $\vec{A}\vec{q}=\vec{n}$:
\begin{gather}
\vec{A}=
\SmallMatrix{
x_1^2 & y_1^2 & z_1^2 & 2x_1y_1 & 2x_1z_1 & 2y_1z_1 & 2x_1 & 2y_1 & 2z_1 & 1\\
x_2^2 & y_2^2 & z_2^2 & 2x_2y_2 & 2x_2z_2 & 2y_2z_2 & 2x_2 & 2y_2 & 2z_2 & 1\\
&&&& \vdots &&&&&\\
2x_1 & 0 & 0 & 2y_1 & 2z_1 & 0 & 2 & 0 & 0 & 0\\
0 & 2y_1 & 0 & 2x_1 & 0 & 2z_1 & 0 & 2 & 0 & 0\\
0 & 0 & 2z_1 & 0 & 2x_1 & 2y_1 & 0 & 0 & 2 & 0\\
2x_2 & 0 & 0 & 2y_2 & 2z_2 & 0 & 2 & 0 & 0 & 0\\
0 & 2y_2 & 0 & 2x_2 & 0 & 2z_2 & 0 & 2 & 0 & 0\\
0 & 0 & 2z_2 & 0 & 2x_2 & 2y_2 & 0 & 0 & 2 & 0\\
&&&& \vdots &&&&&
} \nonumber\\
\vec{n}=\begin{bmatrix}
0 & 0 & \dots & {n}^1_x & {n}^1_y & {n}^1_z & {n}^2_x & {n}^2_y & {n}^2_z & \dots
\end{bmatrix}^T \nonumber\\
\vec{q}=\begin{bmatrix}
A & B & C & D & E & F & G & H & I & J
\end{bmatrix}^T 
\label{eq:system}
\end{gather} 
The matrix $\vec{A}$ is only $4N\times 10$ having identical rank properties as Eq.~\ref{eq:systemfull}. To balance the contribution of normal induced constraints we, introduce a scalar weight $w$. Note that \textit{true} quadric fitting is intrinsically of non-linear nature and hence formulating such a linear form is only possible with algebraic metrics, always being subject to certain bias. We pose it a future work to study the bias introduced by our fit.

Fig. \ref{tab:dualspace} shows the minimum number of constraints required for a full-rank fit in primal and dual spaces.
Besides, to obtain a type-specific fit, a minor redesign of $\vec{A}$ is sufficient.
Fig.~\ref{tab:votingspace} depicts the minimal cases for exemplary shapes and the voting space dimension as we describe hereafter. 
In the following, we elaborate on a more efficient way to calculate a solution to Eq.~\ref{eq:system} rather than using a naive RANSAC on 4 tuples by analyzing its solution space. The rest of this section can be thought of as a generic method applicable to any problem formulated as a linear system.

System ~\ref{eq:system} describes an outlier-free closed form fit. To factor in clutter, a direct RANSAC on 9DoF quadric appears to be trivial. Yet, it has two drawbacks:
1) Evaluating the residuals many times is challenging, as it involves a scene-to-quadric overlap calculation in a geometric meaningful way.
2) Even with the proposed fitting, selecting random 4-tuples from the scene might be slow in practice. 
An alternative to that is Hough voting. However, $\vec{q}$ has $9$ DoFs and is not discretization friendly. The complexity and size of this parameter space makes it hard to construct a voting space directly on $\vec{q}$.
Instead, we now devise a scheme entailing a smaller basis cardinality, with local search. 
\vspace{-3mm}
\paragraph{Parameterizing the solution space}
Let $\vec{q}$ be a solution to Eq.~\ref{eq:system}. Then, $\vec{q}$ can be expressed by a linear combination of particular solution $\vec{p}$ and homogeneous solutions $\bm{\mu}_i$ as:
\begin{align}
\label{eq:nullspace}
&\vec{q}= \vec{p} + \sum_i^D \lambda_i \bm{\mu}_i \\
&= \vec{p} +
\begin{bmatrix}
\bm{\mu}_1 & \bm{\mu}_2 & \cdots
\end{bmatrix}
\begin{bmatrix}
\lambda_1 & \lambda_2 & \cdots
\end{bmatrix}^T = \vec{p} + \vec{N}_A\bm{\lambda}.\nonumber
\end{align}
The dimensionality of the null space $\vec{N}_A$ ($D$) depends on the rank of $\vec{A}$, which is directly influenced by the number of points used: $D=10-rk(\vec{A})$. The exact solution could always be computed by including more points from the scene and validating them, i.e., by a local search. For that reason, the fitting can be split up into distinct parts: First a parametric solution as in Eq. \ref{eq:nullspace} is computed using a subset of points (basis) $\vec{b}=\{\vec{x}_1,...,\vec{x}_m\}$ lying on a quadric. Next, the coefficients $\bm{\lambda}$, and thus the solution, can be obtained by searching for other point(s) $(\vec{x}_{m+1},...,\vec{x}_{m+k})$ which lie on the same surface as $\vec{b}$. 
\vspace{-1mm}
\begin{theorem}
\label{lemma:lambda}
If two point sets $\vec{b}=(\vec{x}_1,...,\vec{x}_m)$ and $\vec{X}=(\vec{x}_{m+1},...,\vec{x}_{m+k})$ lie on the same quadric with parameters $\vec{q}$, then the coefficients $\bm{\lambda}=
\begin{bmatrix}
\lambda_1 & \lambda_2 & \cdots
\end{bmatrix}^T$
of the solution space (\ref{eq:nullspace}) are given by the solution of the system:
\begin{equation}
\label{eq:lemma1}
(\vec{A}_k  \vec{N}_A) \bm{\lambda} = \vec{n}_k - \vec{A}_k \vec{p}
\end{equation}
where $\vec{A}_k$, $\vec{n}_k$ are the linear constraints of the latter set $\vec{X}$ in form of (\ref{eq:system}), $\vec{p}$ is a particular solution and $\vec{N}_A$ is a $10 \times D$ stacked null-space basis as in (\ref{eq:nullspace}), obtained from $\vec{b}$.
\end{theorem}
\begin{proof}\let\qed\relax
Let $\vec{q}$ be a quadric solution for the point set $(\vec{x}_1,...,\vec{x}_m)$ and let $(\vec{A}_k,\vec{n}_k)$ represent the $4 k$ quadric constraints for the $k$ points $\vec{X} = (\vec{x}_{m+1},...,\vec{x}_{m+k})$ in form of (\ref{eq:system}) with the same parameters $\vec{q}$.
\begin{figure}[t!]	
\subfigure[]{
  \setlength{\tabcolsep}{2.5pt}
    \begin{tabular}[b]{lcc}
          & \# Pri. & \# Dual\\
    \toprule
    PD-0 & 9 & 0 \\
    PD-1 & 7 & 1 \\
    PD-2 & 5 & 2 \\
    PD-3 & 4 & 3 \\
     \toprule
    \end{tabular}%
  \label{tab:dualspace}%
}
\hfill
\subfigure[]{
\setlength{\tabcolsep}{2.5pt}
\begin{tabular}[b]{lcccc}
          & \# Pri. & \# Dual & $|\vec{b}|$ & VS\\
    \toprule
    Plane & 1 & 1 & 1 & 0\\
    2-Planes & 2 & 2 & 2 & 0 \\
    Sphere & 2 & 1 & 1 & 1\\
    Spheroid & 2 & 2 & 2 & 3 \\
     \toprule
    \end{tabular}%
  \label{tab:votingspace}%
}
\caption{\textbf{(a)} Number of constraints for a minimal fit in Primal(P) or Dual(D) spaces. PD-i refers to $\text{i}^{\text{th}}$ combination. \textbf{(b)} Number of minimal constraints and voting space size for various quadrics.} 
\end{figure}
As $\vec{x}_i \in \vec{X}$ by definition lies on the same quadric $\vec{q}$, it also satisfies $\vec{A}_k \vec{q} = \vec{n}_k$.
Inserting Eq. (\ref{eq:nullspace}) into this, we get:
\begin{align}
\vec{A}_k (\vec{p}+\vec{N}_A\bm{\lambda}) &= \vec{n}_k\\
(\vec{A}_k  \vec{N}_A) \bm{\lambda} &= \vec{n}_k - \vec{A}_k \vec{p}.
\label{eq:lambdas2}
\end{align}
\end{proof}
Solving Eq. (\ref{eq:lambdas2}) for $\bm{\lambda}$ requires a multiplication of a $4 k \times 10$ matrix with a $10 \times m$ one and ultimately solving a system of $4 k$ equations in $m$ unknowns. This is much more efficient for $k < m$ than re-solving the system (\ref{eq:system}) and resembles updating the solution online for a stream of points.
\vspace{-3mm}
\paragraph{Local voting for quadric detection}
Given a fixed basis composed of $b$ points $(b>0)$ as in Fig. \ref{fig:basis}, a parametric solution can be described. The actual solution can then be found quickly as explained above by incorporating new points lying on the same quadric. Thus, the problem of quadric detection is de-coupled into
1) finding a proper basis and
2) searching for compatible scene points. In this part, we assume the basis is correctly found and explain the search by voting. For a fixed basis $\vec{b}_i$ on a quadric, we form the null-space decomposition of the under-determined system $\vec{A}_i\vec{q}=\vec{n}_i$. We then sample further points from the scene and compute the required coefficients $\bm{\lambda}$. Thanks to Thm. \ref{lemma:lambda}, this can be done efficiently. Sample points lying on the same quadric as the basis (inliers) generate the same $\bm{\lambda}$ whereas outliers will produce different values. Therefore we propose to construct a voting space on $\bm{\lambda}$ and cast votes to maximize the local consensus (Fig.~\ref{fig:basis}). The size of the voting space is a design choice and depends on the size of $\vec{b}_i$ vs. the DoFs desired to be recovered (see Fig.~\ref{tab:votingspace}). We find from Fig.~\ref{tab:dualspace} that using a 3-point basis is advantageous for a generic quadric fit due to 1D search space.
\insertimageC{1}{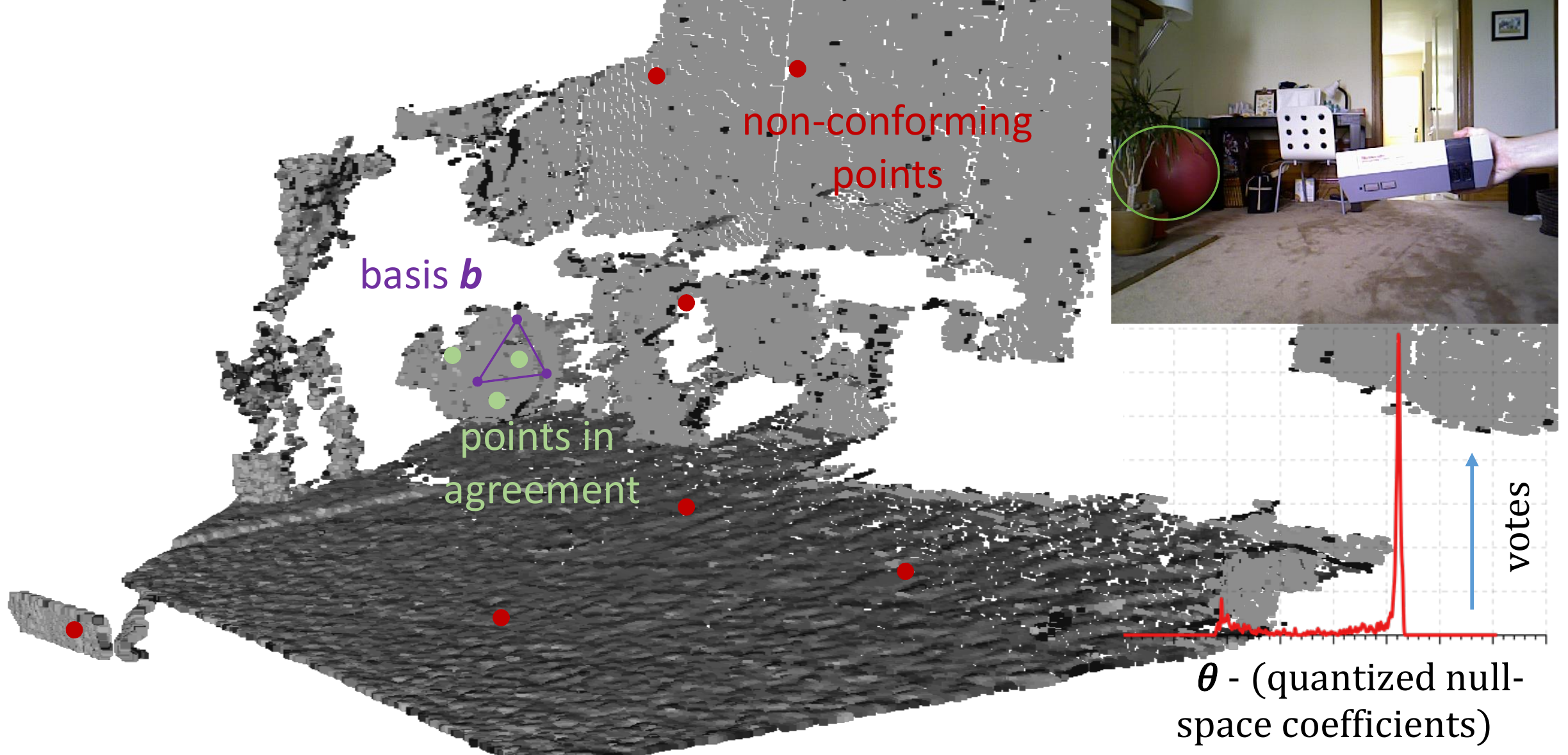}{Once a basis is randomly hypothesized, we look for the points on the same surface by casting votes on the null-space. The sought pilates ball (likely quadric) is marked on the image and below that lies the corresponding filled accumulator by KDE~\cite{rosenblatt1956remarks}.}{fig:basis}{t!}
\vspace{-3mm}
\paragraph{Efficient computation of voting parameters}
For a 3-point basis, adding a fourth sample point $\vec{x}_4$ completes $rk(\vec{A}) = 10$ and a unique solution can be computed.
Yet, as we will select multiple $\vec{x}_4$ candidates per basis, hypothesized in a RANSAC loop, an efficient scheme is required, i.e. it is undesirable to re-solve Eq. \ref{eq:system} for each incoming $\vec{x}_4$ tied to the basis.
It turns out that once again, the solution can be obtained directly from Eq. (\ref{eq:nullspace}):
\begin{theorem}
\label{lemma:lambda_one}
If the null-space is one dimensional (with only 1 unknown) it holds $\vec{N}_A\bm{\lambda} = \lambda_1 \bm{\mu}_1$ and the computation in Thm. \ref{lemma:lambda} reduces to the explicit form:
\begin{align}
\label{eq:lambdas_fast}
\lambda_1 = ({\vec{A}_1\vec{N}_A})^T/{\|\vec{A}_1\vec{N}_A\|^2} \cdot (\vec{n}_1 - \vec{A}_1\vec{p})
\end{align}
\end{theorem}
\begin{proof}
Let us re-write Eq. \ref{eq:lambdas2} in terms of the null space vectors: $\lambda_1(\vec{A}_1\bm{\mu}_1) = \vec{n}_1 - \vec{A}_1 \vec{p}$.
A solution $\lambda_1$ can be obtained via Moore-Penrose pseudoinverse~\cite{Moore1920} as $\lambda_1 = (\vec{A}_1\bm{\mu}_1)^{+}(\vec{n}_1 - \vec{A}_1 \vec{p})$.
Because for 1-dimensional null spaces $\vec{A}_1\bm{\mu}_1$ is a vector, ${}^{+}$ operator is defined as: $\bm{v}^{+}=\bm{v}^T / (\bm{v}^T\bm{v})$. Substituting this in Eq. \ref{eq:lemma1} gives Eq. \ref{eq:lambdas_fast}.
\end{proof}
\vspace{-2mm}
Thm. \ref{lemma:lambda_one} enables a very quick computation of the parameter hypothesis in case of 1 point. A MATLAB implementation takes ca. $30 \mu s$ per $\lambda$. Besides, for a 3-point method, inclusion of only a single equation is sufficient, letting the normal of the $4^{th}$ point remain unused and amenable for verification of fit. We only accept to vote a candidate if the gradient of the fit quadric agrees with the normal at $4^{th}$ point: $({\nabla \vec{Q}(\vec{x}_4)}/{\| \nabla \vec{Q}(\vec{x}_4) \|}) \cdot \vec{n}(\vec{x}_4) > \tau_n$.
\vspace{-1mm}
\paragraph{Quantizing $\bm{\lambda}$}
To vote on $\bm{\lambda}$, we need a linear, bounded and discrete space. Yet, $\bm{\lambda}$ is not quantization-friendly, as it is unbounded and has a non-linear effect on the quadric shape (see Fig. \ref{fig:lambdas}).
Thus, we seek to find a geometrically meaningful transformation to a bounded and well behaving space so that quantization would lead to little bias and artifacts. From a geometric perspective, each column of $\vec{N}_A$ in Eq. \ref{eq:nullspace} is subject to the same coefficient $\vec{\lambda}$, representing the slope of a high dimensional line in the solution space. Thus, it could as well be viewed as a rotation. For 1D null-space, we set: $\theta = \text{atan2}\big(\,{(y_2-y_1)}\, , \,{(x_2-x_1)}\,\big)$, where $[x_1,y_1,\cdots]^T=\vec{p}$ and $[x_2,y_2 \cdots ]^T$ is obtained by moving in the direction $\vec{N}_A$ from the particular solution $\vec{p}$ by an offset $\lambda$.\footnote{Simple $tan^{-1}(\lambda)$ could work but would be more limited in the range.} This new angle $\theta$ is bounded and thus easy to vote for. As the null-space dimension grows, $\vec{\lambda}$ starts to represent hyperplanes, still preserving the geometric meaning, i.e. for $d>1$, different $\bm{\theta}=\{\theta_i\}$ can be found.
\begin{figure*}[t!]
\begin{center}
\subfigure[Effect of $\lambda$ on the surface geometry.]{
\includegraphics[width=0.64505\linewidth]{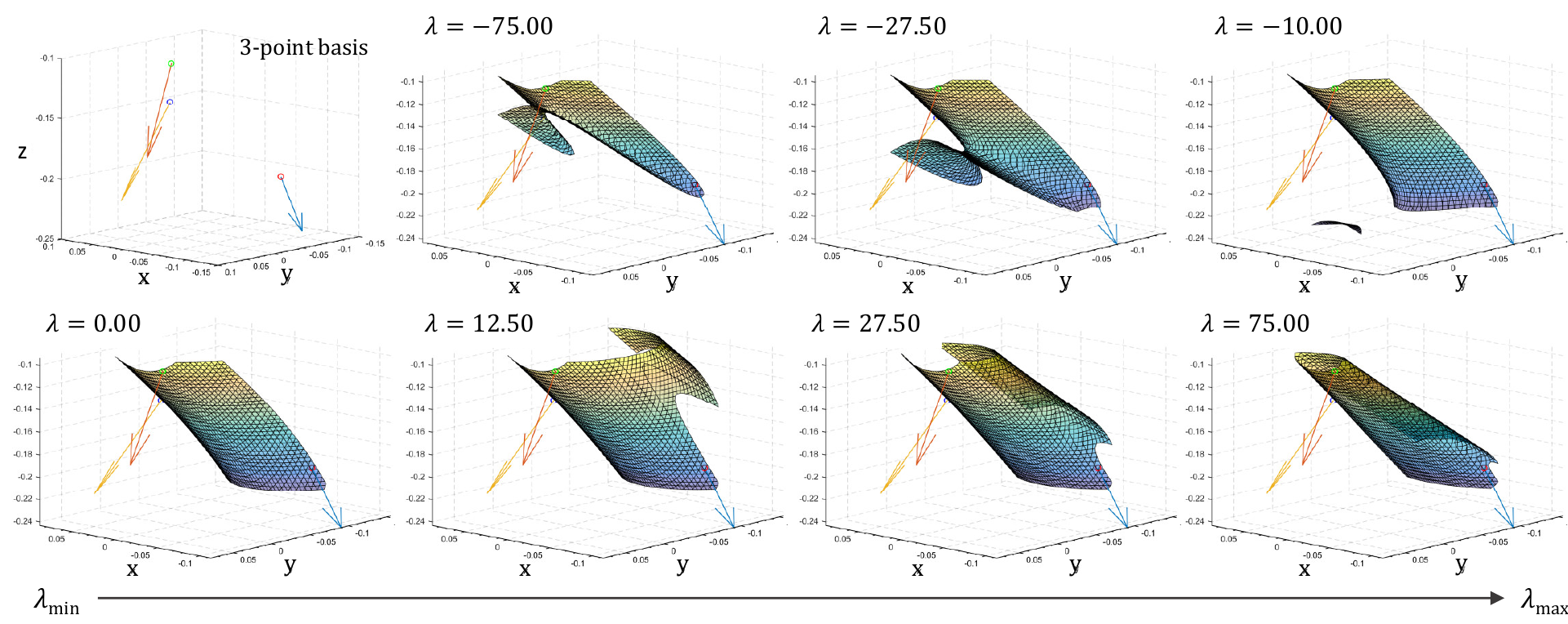} 
\label{fig:lambdas}
}
\subfigure[Distributions related to $\lambda$.]{
\includegraphics[width=0.3215\linewidth]{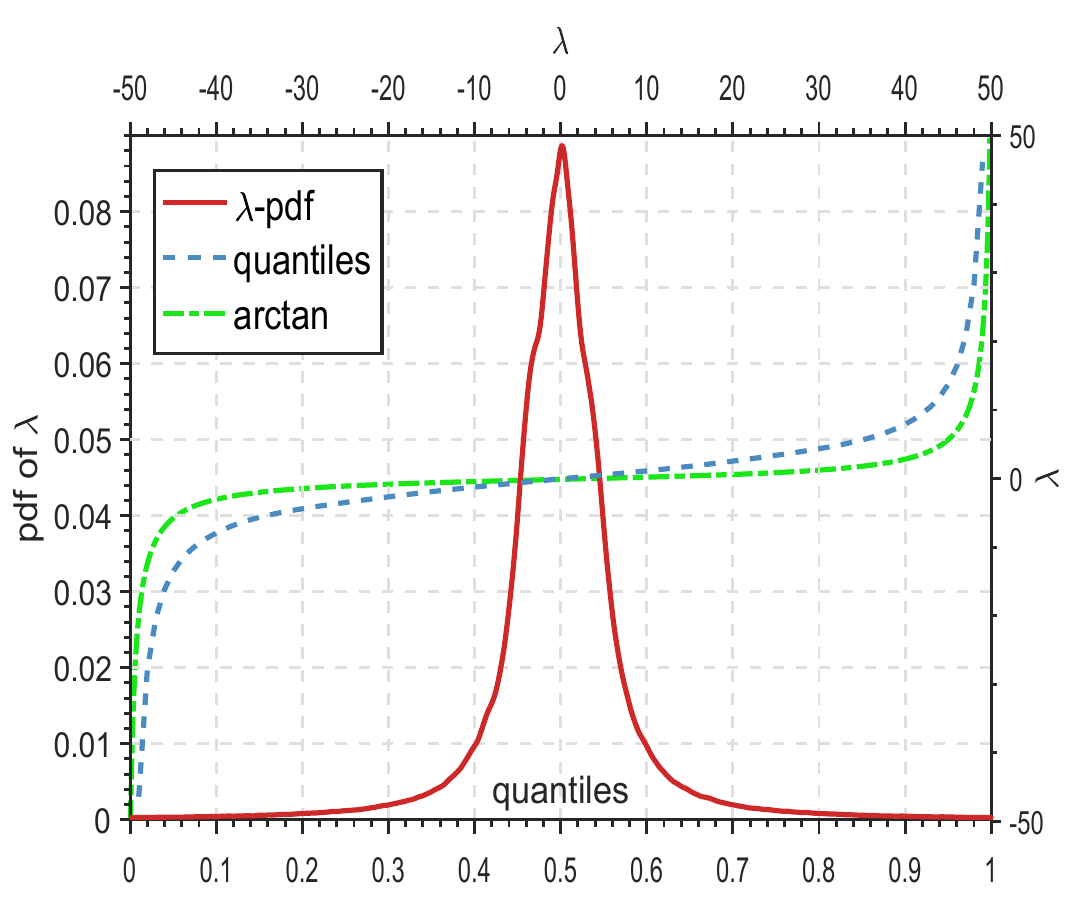}
\label{fig:lambda_pdf}
}
\end{center}
   \caption{Characteristics of $\bm{\lambda}$. \textbf{(a)} We compute null-space decomposition for a fixed basis and vary $\lambda$ from -75 to 75 to generate different solutions $\vec{q}$ along the line in the solution space. The plot presents the transition of the surface controlled by $\lambda$. \textbf{(b)} Statistical distribution of the solution-space coefficient and our quantization function: PDF (red curve) and inverse CDF (dashed blue-curve) of $\lambda$ over collected data, and $tan^{-1}$ function (green-line). Note that our quantization function also well aligns with practice. Smaller $|\lambda|$ results in finer binning.}
\end{figure*}
\vspace{-1mm}
\paragraph{Hypotheses Aggregation} Up to this point, we have described how to find plausible quadrics given local triplet bases. As mentioned, to discover the basis lying on the surface, we employ RANSAC~\cite{Fischler1981}, where each triplet might generate a hypothesis to be verified. Many of those will be similar as well as dissimilar. 
Thus, the final stage of the algorithm aggregates the potential hypotheses to reduce the number of candidate surfaces and to increase the per quadric confidence.
Not to sacrifice further speed, we run an agglomerative clustering similar to~\cite{birdal2015point} in a coarse to fine manner: 
First an algebraic but fast distance measure helps to cluster the obvious hypotheses:
\begin{align}
d_{close}(\vec{Q}_1, \vec{Q}_2) 
:= \mathbbm{1}( \|\vec{q}_1-\vec{q}_2 \|_1 < \tau) 
\cdot \| \vec{Q}_1 \vec{Q}_2^{+} - \vec{I} \|_F 
\label{eq:qdclose}
\end{align}
where $\vec{I} \in \mathbb{R}^{4 \times 4}$ is the identity matrix and $\mathbbm{1}: \mathbb{R} \rightarrow \left\{0, 1\right\}$ the indicator function. Second, a pseudo-geometric one is executed only on a subset of the surviving quadrics:
\begin{align}
&d_{far}(\vec{Q}_1, \vec{Q}_2):= \\
&1-\frac{1}{K}\sum\limits_{i=1}^{K}
\mathbbm{1}\left(\lvert\vec{x}_i^T\vec{Q}_1\vec{x}_i\rvert<\tau\right) \cdot
\mathbbm{1}\left(\lvert\vec{x}_i^T\vec{Q}_2\vec{x}_i\rvert<\tau\right) \cdot \nonumber\\
& \mathbbm{1}\left(1-\vec{n}_i \cdot \nabla\vec{Q}_1(\vec{x}_i) <\tau_n\right) \cdot
\mathbbm{1}\left(1-\vec{n}_i \cdot \nabla\vec{Q}_2(\vec{x}_i) <\tau_n\right)\nonumber
\label{eq:qdfar}
\end{align}
where the scene is composed of $K$ points $\{\vec{x}_i\}$. Each quadric is associated a score proportional to the number of points with compatible normals, found to lie on it. This score is used to re-rank the detections.

\section{Experimental Evaluation and Discussions}
\insertimageStar{1}{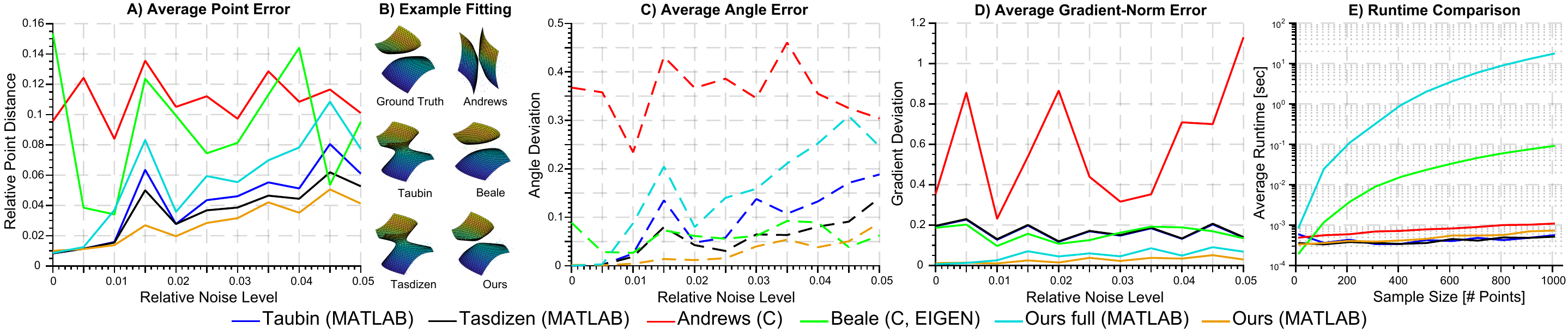}{Synthetic evaluations. \textbf{(A,C)} 
Mean geometric point(A) and angle(C) errors for different quadric fitting methods.
The per point error is measured as the average point-to-mesh distance from every gnd. truth vertex to the fitted quadric.
The angular error (dashed) is computed as the negated dot product between quadric gradient and the gnd. truth normal.
\textbf{(B)} gives exemplary fitting results while \textbf{(D)} shows the average error of the gradient norm compared to the ground truth.
\textbf{(E)} plots the speed and detection rate on synthetic data.}{fig:synth_exp}{t!}
\paragraph{Implementation details}
Planes are singular cases and occupy large spaces of 3D scenes.
A preprocessing removes them using our type specific plane fit. Normalizing the points to a unit ball is also necessary to increase the numerical stability~\cite{hartley1997defense}.
Next, we downsample the scene using a spatial voxel-grid~\cite{birdal2017iros} enforcing a minimum distance of $\tau_s \cdot diam(\vec{X})$ between the samples ($0.025 \leq \tau_s \leq 0.05$) and compute the normals~\cite{Hoppe1992} on this subset.
What follows is an iterative selection of triplets from this reduced set.
At each iteration, once $\vec{x}_1$ is fixed, we query the points which are likely to be combined with $\vec{x}_1$ to form the 3-point basis $\vec{b}$.
The rest of the points are then randomly selected from this subset.
Triplets which are found to be collinear or which do not result in a rank-9 system are directly removed. We also hash the triplets such that any basis is seen only once. 
\vspace{-3mm}
\paragraph{Synthetic tests of fitting}
To asses the accuracy of the proposed fitting, we generate a synthetic test set of multitudes of random quadrics and compare our method with the fitting procedures of Taubin~\cite{Taubin1991}, Tasdizen~\cite{Tasdizen2001}, Andrews~\cite{Andrews2014}, and Beale~\cite{beale2016fitting} as well as our full, least squares fit.
Gaussian noise with $\sigma=[0\% - 5\%]$ relative to the size $s$ of the quadric is added to the vertices before the methods are run.
At each noise level, 10 random quadrics are tested.
We report the average point-to-mesh distance and the angle deviation as well as the runtime performances in Fig. \ref{fig:synth_exp}.
For the constrained fitting method~\cite{Andrews2014} we pre-specified the type, which might not be possible in a real application. We perform not single but 20 fits per set. 
Although, our fit is designed to use a minimal number of points, it also proves to be robust when more points are added and is among the top fitters for the distance \& angle errors. This also demonstrates that gradient alignment cannot be surpassed by simply inserting more primal points. Fig.~\ref{fig:synth_exp}{\color{red}D} shows that norms of the gradient of our quadrics also align well with the ground truth, favoring the validity of our approximation. 
Our full method performs the best on low noise levels but quickly destabilizes. This is because the system is biased to compute correct norms and it has increased parameters.
We believe the reason for our compact fit to work well is the soft constraint of common scale factor acting as a weighted regularizer towards special quadrics. When this constraint cannot be satisfied, the solution settles for an acceptable shape.
We also found that for a visually appealing fit, the normal alignment is crucial which is why the closest shapes arise from our method and~\cite{beale2016fitting}. The latter suffers from a higher point error due to the use of the strong ellipsoid prior.
\begin{figure}[t!]
\centering
\subfigure[\textbf{(i)} Images captured by an industrial structured light sensor (1-5) and Kinect (last image). \textbf{(ii)} Corresponding 3D scene. \textbf{(iii)} Detected quadric, shown without clipping. \textbf{(iv)} Quadric in (iii) clipped to the points it lies on.]{
\includegraphics[width=\columnwidth, clip=true]{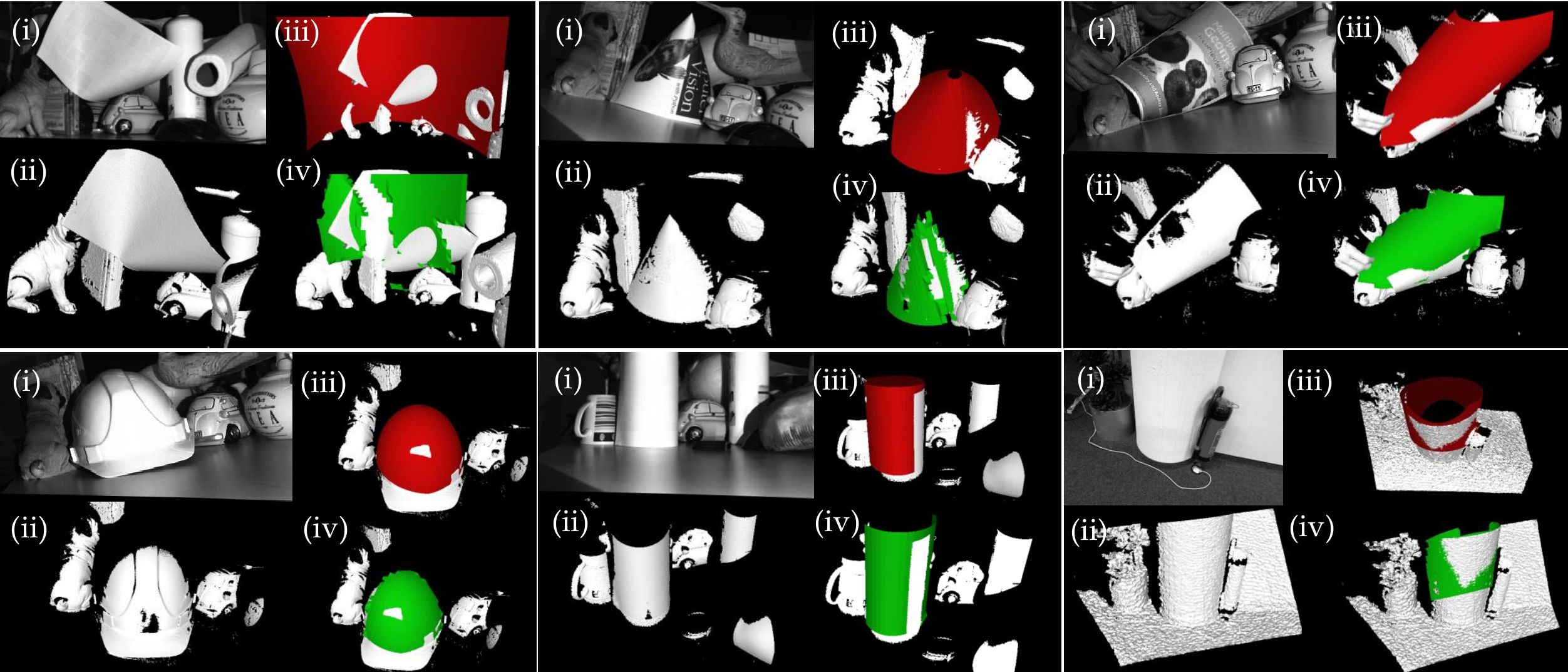}
\label{fig:qualitative}}
\subfigure[Speed and detection results.]{
\includegraphics[width=0.48815\columnwidth, clip=true]{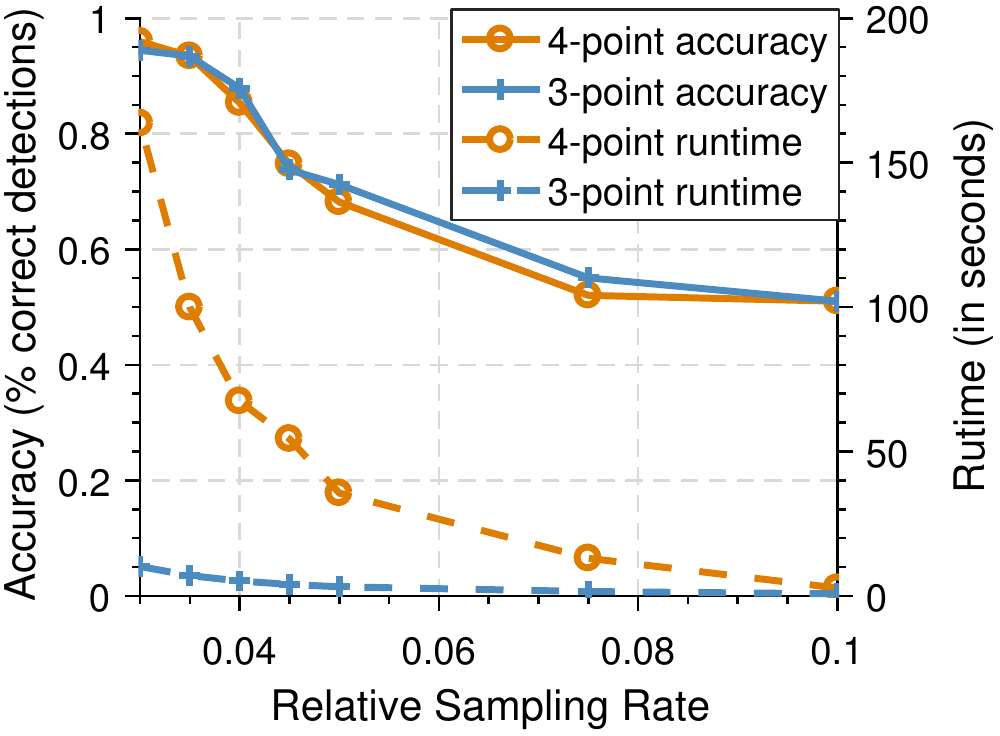}
\label{fig:real}}
\hfill
\subfigure[Distribution of errors of the fit.]{
\includegraphics[width=0.46\columnwidth, clip=true]{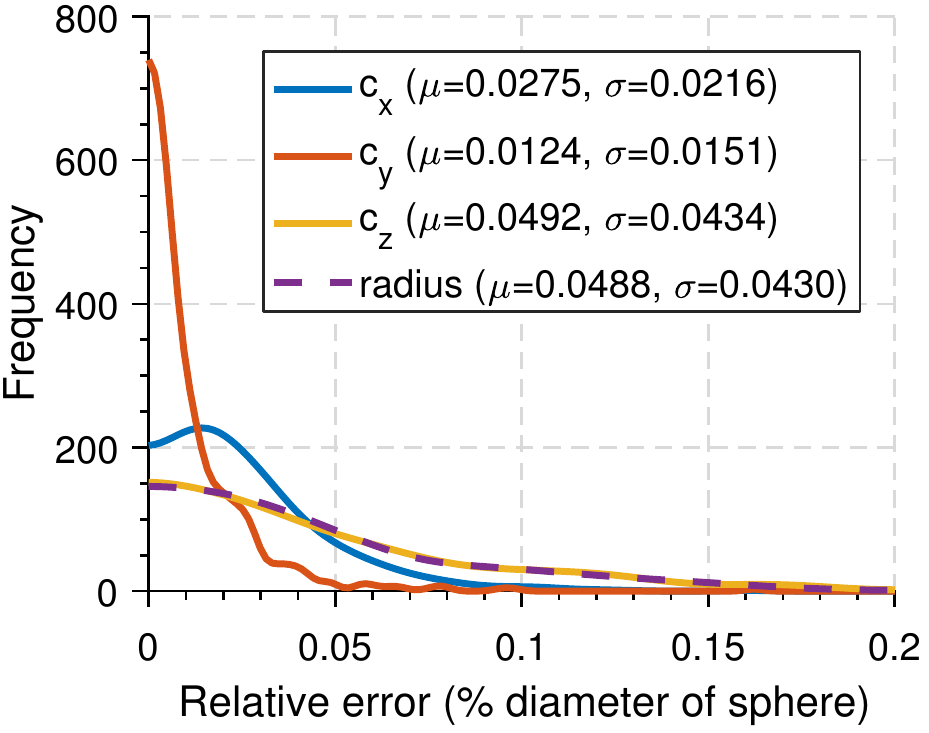}
\label{fig:poseacc}}
\caption{Experiments on real datasets.}
\vspace{-0.5mm}
\end{figure}
\vspace{-5mm}
\paragraph{Is $\atantwo$ a valid transformation for $\bm{\lambda}$?}
To assess the practical validity of the quantization, we collect a set of 2.5 million oriented point triplets from several scenes and use them as bases to form the underdetermined system $\vec{A}$. We then sample the $4^{th}$ point from those scenes and compute $\lambda$. Next, we establish the probability distribution $p(\lambda)$ for the whole collection to calculate the quantiles, mapping $\lambda$ to bins as the inverse CDF. A similar procedure has been applied to cross ratios in~\cite{Birdal2016}. We plot the findings together with the $\atantwo$ function in Fig. \ref{fig:lambda_pdf} and show that empirical distribution and $\atantwo$ follow a similar trend, justifying that our quantizer is well suited to the data behavior. 
\vspace{-3mm}
\paragraph{Quadric detection on a real dataset} 
There is no well-established dataset of real scenes for quadric detection.
Therefore we collect our own set.
We use a phase-shift stereo structured light scanner and capture 35 3D scenes of 5 different objects within clutter and occlusions.
Our objects are bending papers, helmet, paper towel and cylindrical spray bottle.
Other objects are included to create clutter.
To obtain the ground truth, for each scene, we generated a visually acceptable set of quadrics using 1) \cite{Schnabel2007} when shapes represent known primitives 2) by segmenting the cloud manually and performing a fit when shapes are more general.
Each scene then contains 1-3 ground truth quadrics.
To assess the detection accuracy, we manually count the number of detected quadrics aligning with the ground truth.
We compare the 4-point and 3-point algorithms, both of which we propose.
We also tried the naive 9-point RANSAC algorithm (with \cite{Taubin1991}), but found it to be infeasible when the initial hypotheses of the inlier set is not available. Fig. \ref{fig:qualitative} visualizes the detected quadrics both on our dataset and on a 3D data captured by Kinect v1.
Fig. \ref{fig:real} presents our accuracy over different sampling rates along with the runtime performance.
It is clear that our 3-point method is on par with the 4-point variant in terms of detection accuracy, while being significantly faster.
We also evaluate our detector on the large objects dataset of \cite{Choi2016} without further tuning. Tab. \ref{tab:accuracy} shows $100\%$ accuracy in locating a frontally appearing ellipsoidal rugby ball over a $1337$ frame sequence without type prior. While such scenes are not particularly difficult, it is noteworthy that we manage to generate the similar quadric repeatedly at each frame within $5\%$ of the quadric diameter.
\begin{table}[t!]
  \centering
  \caption{Detection accuracy on real datasets.}
  \setlength\tabcolsep{2 pt}
  \resizebox{\columnwidth}{!}{
    \begin{tabular}{lccccc}
    \toprule
          & Dataset & \# Objects & \multicolumn{1}{c}{Type} & Occlusion & \multicolumn{1}{c}{Accuracy} \\
    \midrule
    Pilates Ball 1 & Ours  & 580   & \multicolumn{1}{c}{Generic} & Yes   & \multicolumn{1}{c}{94.40\%} \\
    Rugby Ball & \cite{Choi2016}    & 1337  & \multicolumn{1}{c}{Generic} & No    & \multicolumn{1}{c}{100.00\%} \\
    Pilates Ball 2 & \cite{Choi2016}    & 1412  & \multicolumn{1}{c}{Sphere} & Yes   & \multicolumn{1}{c}{100.00\%} \\
    Big Globe & \cite{Choi2016}    & 2612  & \multicolumn{1}{c}{Sphere} & Yes   & \multicolumn{1}{c}{90.70\%} \\
    Small Globe & \cite{Choi2016}    & 379   & \multicolumn{1}{c}{Sphere} & Yes   & \multicolumn{1}{c}{56.90\%} \\
    Apple & \cite{Choi2016}    & 577   & \multicolumn{1}{c}{Sphere} & Yes   & \multicolumn{1}{c}{99.60\%} \\
    Football & \cite{Choi2016}    & 1145  & \multicolumn{1}{c}{Sphere} & Yes   & \multicolumn{1}{c}{100.00\%} \\
    Orange Ball & \cite{Choi2016}    & 270   & \multicolumn{1}{c}{Sphere} & Yes   & \multicolumn{1}{c}{93.30\%} \\
    \bottomrule
    \end{tabular}%
    }
  \label{tab:accuracy}%
\end{table}%
\insertimageStar{1}{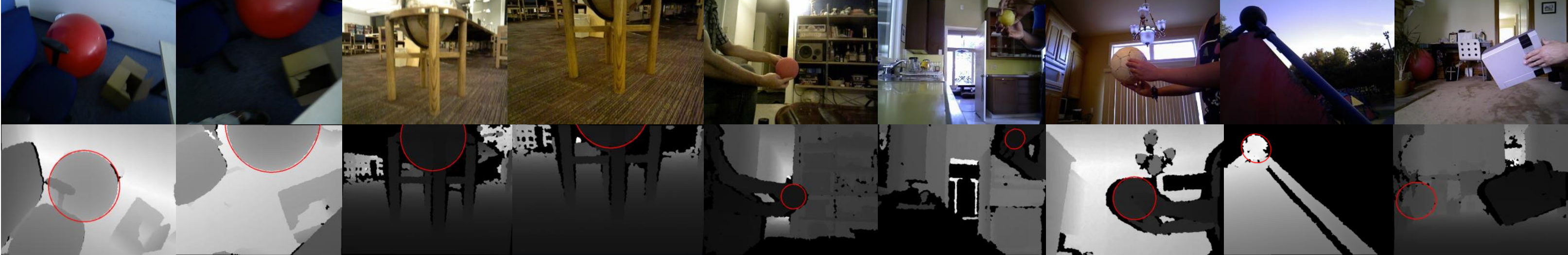}{Qualitative visualizations of sphere detection in the wild: Our algorithm is very successful in difficult scenarios including clutter, occlusions, large distances. Note that the sphere is detected in 3D only using the point clouds of depth images and we draw the apparent contour of the quadric. The RGB pictures are also included in the top row to ease the visual perception.}{fig:spherefit}{t!}
\vspace{-3mm}
\paragraph{How fast is it?}
As our speed is influenced by the factors of closed form fitting, RANSAC and local voting, we evaluate the fit and detection separately. Fig. \ref{fig:synth_exp}{\color{red}E} shows the runtime of fitting part. Our method scales linearly due to the solution of an $4N\times 10$ system, but it is the fastest approach when $<300$ points are used. Thus, it is more preferred for a minimal fit. Fig. \ref{fig:real} then presents the order of magnitude speed gain, when our 4-point C++ version is replaced by the 3-point without accuracy loss. Although the final runtime is in the range of 1-2 seconds, our 3-point algorithm is still the fastest known method in segmentation free detection.
\vspace{-3mm}
\paragraph{How accurate is the fit?}
To evaluate the pose accuracy on real objects, we use closed geometric objects of known size from the datasets and report the distribution of the errors, and its statistics. We choose \textit{football} and \textit{pilates ball 1} as it is easy to know their geometric properties (center and radius). We compare the radius to the true value while the center is compared to the one estimated from a non-linear refinement\footnote{We only use Gauss-Newton non-linear quadric refinement for evaluations of accuracy and never for the actual surface estimation.} of the sphere. Our results are depicted in Fig. \ref{fig:poseacc}. Note that the errors successfully remain about the used sampling rates, which is as best as we could get.
\vspace{-3mm}
\paragraph{Type-specific detection} 
It is remarkably easy to convert our algorithm to a type specific detection by re-designing matrix $\vec{A}$. We evaluate a sphere-specific detection on scenes from~\cite{Choi2016} which contains spherical everyday objects. Tab.~\ref{tab:accuracy} summarizes the dataset and reports our accuracy while Fig.~\ref{fig:spherefit} qualitatively shows that our sphere-specific detector can indeed operate in challenging real scenarios. Due to reduced basis size ($b=1$) this type specific fit can meet real-time criteria, operating in $\sim 27\text{ms}$ on an Intel i5 CPU.
\vspace{-3mm}
\insertimageC{1}{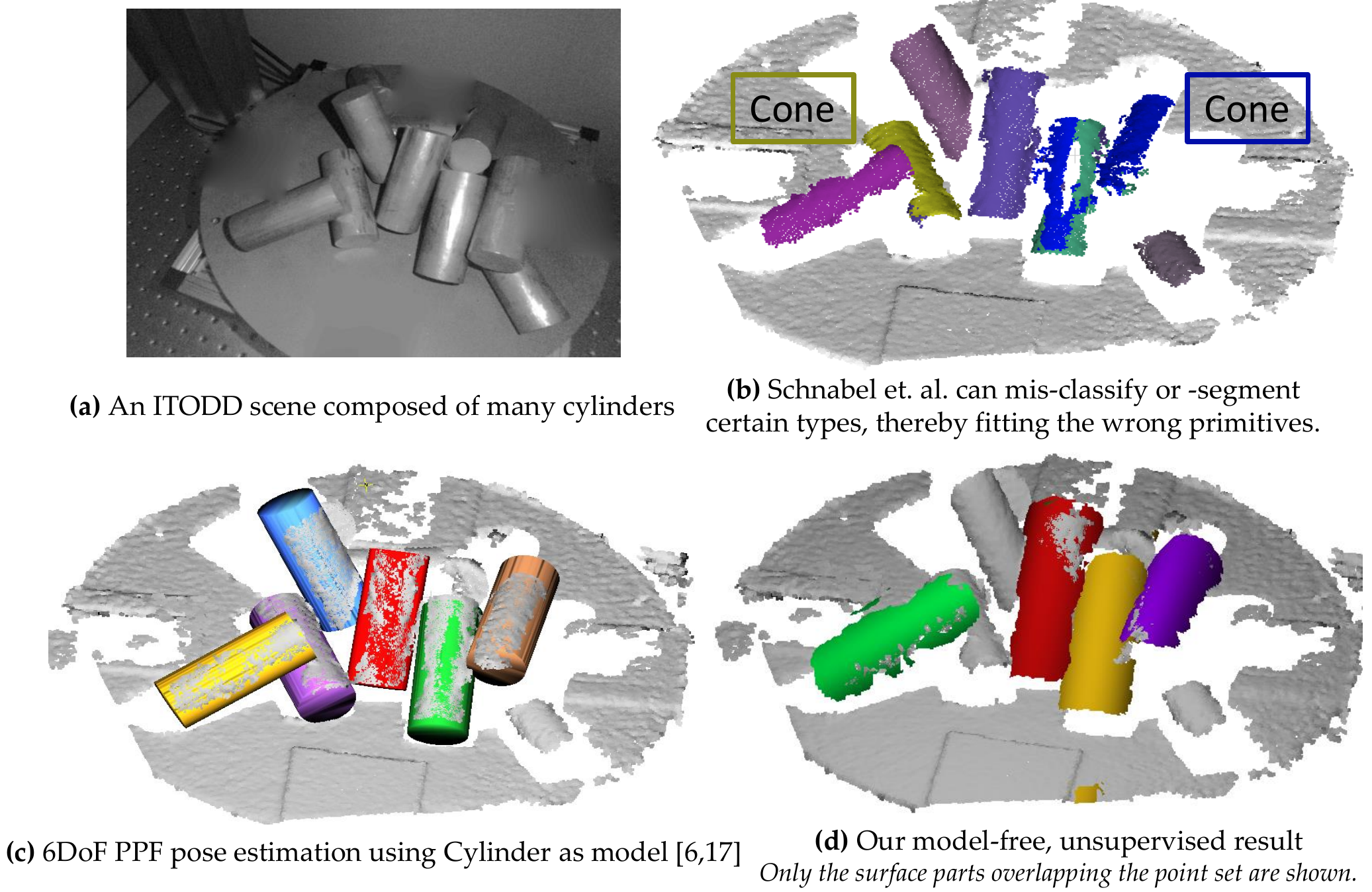}{Multiple cylinder detection in clutter and occlusions: Our approach is type agnostic and uninformed about cylinders.}{fig:cylinders}{t!}
\vspace{-3mm}
\paragraph{Comparison to model based detectors} 
The literature is overwhelmed by the number of 3d model based pose estimation methods. Hence, we decide to compare our model-free approach to the model based ones. For that, we take the cylinders subset of the recent ITODD dataset~\cite{Drost2017} and run our generic quadric detector without training or specifying the type. Visual results of different methods are presented in Fig. \ref{fig:cylinders} whereas detection performance are reported in Tab.~\ref{tab:cylinders}. Our task is not to explicitly estimate the pose. Thus, we manually accept a hypothesis if ICP~\cite{Fitzgibbon2003} converges to a visually pleasing outcome. Note, multiple models are an important source of confusion for us, as we vote on generic quadrics. However, our algorithm outperforms certain detectors, even when we are solving a more generic problem as our shapes are allowed to deform into geometries other than cylinders. 
\begin{table}[t!]
  \centering
  \caption{Results on ITODD~\cite{Drost2017} cylinders: Even without looking for a cylinder, we can do better than the model based ~\cite{Ulrich2012}.}
  \setlength\tabcolsep{2 pt}
  \resizebox{\columnwidth}{!}{
    \begin{tabular}{cccccc}
    \toprule
    PPF3D & PPF3D-E & PPF3D-E-2D & S2D~\cite{Ulrich2012}   & RANSAC~\cite{Papazov2010} & Ours \\
    \midrule
72\% & 73\% & 74\% & 24\% & 86\% & 41.9\% \\ 
    \bottomrule
    \end{tabular}%
  \label{tab:cylinders}%
  }
\end{table}%
\vspace{-3mm}
\paragraph{Cons} 
Unless made specific, our method is surpassed by type-specific fits in detection rate.
However, we detect a larger variety of shapes.
Yet, if $\bm{A}$ targets a specific type, we perform even better. Degenerate cases are also difficult for us, but we find a close approximation - see Fig. \ref{fig:planes}.
\insertimageC{1}{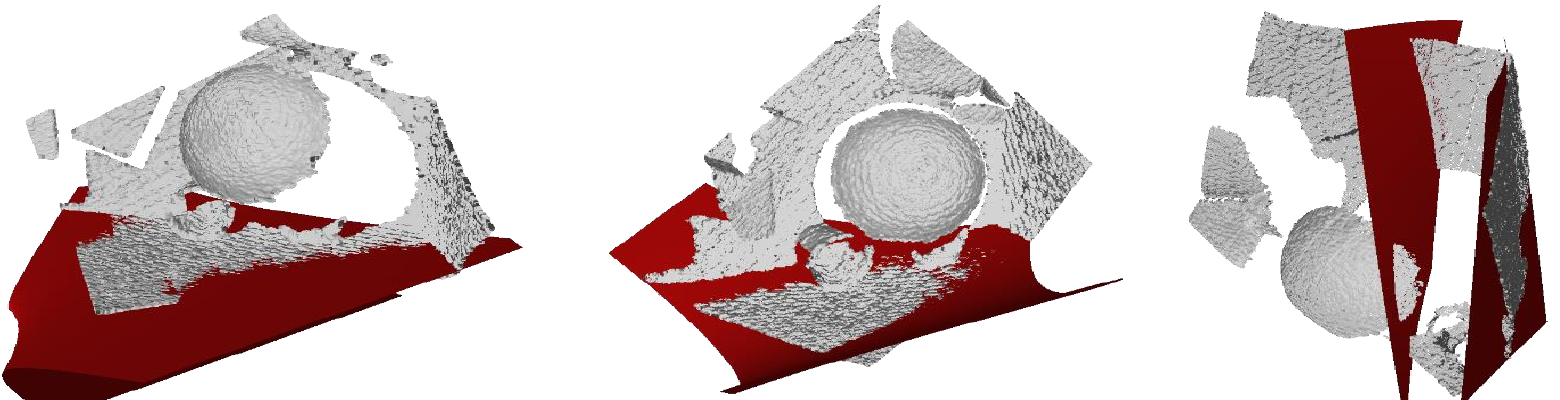}{When planes gather a majority of the votes, our method approximates them by a close non-degenerate quadric surface.}{fig:planes}{h}
\vspace{-1.5mm}
\section{Conclusion}
We presented a fast and robust approach for generic primitive detection in noisy / cluttered 3D scenes. In a nutshell, a novel, linear fitting formulation for oriented point quadruplets is backed by an efficient null-space voting, potentially paving the way towards real-time operation. We establish the minimalist case known up to now: 3 oriented points. Optionally, we can convert to a type-specific fit to boost speed and accuracy. Future work will address the enlisted limitations and further analysis of the bias in the fit.

{
\bibliographystyle{ieee}
\bibliography{egbib}
}

\end{document}